\documentclass{article}
\usepackage{graphicx}
\usepackage{hyperref}
\usepackage{amsmath,amsfonts,amssymb,amsthm}
\usepackage{cleveref}
\usepackage{enumerate}
\usepackage[sortcites, hyperref]{biblatex}
\usepackage{geometry}
\usepackage{xcolor}
\usepackage{algorithm2e}
\usepackage{bbm}
\usepackage{thm-restate}
\usepackage{mathabx}
\usepackage{thmtools}
\usepackage{csquotes}
\usepackage{authblk}
\usepackage{algpseudocode}
\usepackage{setspace}
\usepackage{fontenc}
\usepackage{booktabs} 
\usepackage{tikz}
\usepackage[most]{tcolorbox}

\global\long\def\poly{\mathrm{poly}}

\global\long\def\supp{\mathrm{supp}\,}
\DeclareMathOperator{\findheavyfourier}{\textsc{FindHeavyFourier}}
\DeclareMathOperator{\findcandidate}{\textsc{FindCandidate}}
\DeclareMathOperator{\findset}{\textsc{FindSet}}
\DeclareMathOperator{\learnfouriercoefficients}{\textsc{LearnFourier}}
\DeclareMathOperator{\tester}{\textsc{Tester}}

\newcommand{\hP}{\hat \cP}
\newcommand{\hD}{\hat \cD}

\newcommand{\hQ}{\hat \cQ}
\newcommand{\hZ}{\hat \cZ}

\newcommand{\hR}{\hat \cR}

\newcommand{\boolset}{\{\pm 1\}}
\newcommand{\zerooneset}{\{0, 1\}}
\renewcommand{\hat}{\widehat}

\usepackage{dsfont}

\newcommand*\ie{i.\kern.1em e.\ }
\newcommand*\eg{e.\kern.1em g.\ }
\newcommand*\cf{c.\kern.1em f.\ }
\newcommand*\almev{a.\kern.1em e.\ }

\newtheorem*{theorem*}{Theorem}
\newtheorem{lemma}{Lemma}[section]

\newtheorem{definition}[lemma]{Definition}

\theoremstyle{definition}

\newtheorem{observation}{Observation}

\theoremstyle{plain}

\newcommand{\ld}{\left}
\newcommand{\rd}{\right}

\newcommand{\cE}{\mathcal E}
\newcommand{\cJ}{\mathcal J}
\newcommand{\cU}{\mathcal U}

\newcommand{\bbR}{\mathbb R}
\newcommand{\cP}{\mathcal P}
\newcommand{\cQ}{\mathcal Q}
\newcommand{\cF}{\mathcal F}
\newcommand{\cD}{\mathcal D}

\newcommand{\cZ}{\mathcal Z}

\newcommand{\cA}{\mathcal A}

\newcommand{\cI}{\mathcal I}
\newcommand{\cR}{\mathcal R}

\newcommand{\one}{\mathbbm 1}

\newcommand{\bE}{\ensuremath{\mathbb{E}}}

\newcommand{\bP}{\ensuremath{\mathbb{P}}}

\renewcommand{\epsilon}{\varepsilon}

\DeclareMathOperator{\ber}{Ber}

\newcommand{\Var}[1]{\mathrm{Var} \left[ #1 \right]}

\newcommand{\Ex}[1]{\bE \left[ #1 \right]}
\newcommand{\Exui}[2]{\bE_{#1} \left[ #2 \right]}
\newcommand{\Exu}[2]{\underset{#1} \bE \left[ #2 \right] }

\renewcommand{\Pr}[1]{\bP \left[ #1 \right]} 
\newcommand{\Pru}[2]{\underset{ #1 }\bP \left[ #2 \right]}
\newcommand{\Prui}[2]{\bP_{#1} \left[ #2 \right]}

\newcommand{\lpnorm}[2]{\left|\left| #2 \right |\right |_{ #1 }}
\newcommand{\abs}[1]{\left| #1 \right |}

\newcommand{\E}{\bE}

\definecolor{ama-iro}{RGB}{0, 158, 243.0}
\definecolor{fuyu-gaki}{RGB}{251, 74, 52}
\definecolor{momiji}{RGB}{245, 70, 111}
\definecolor{hotaru-bi}{RGB}{229,221,58} 
\definecolor{kon-peki}{RGB}{1,120,217}
\definecolor{shin-kai}{RGB}{77,98,152}
\definecolor{shin-ryoku}{RGB}{1,145,97}
\definecolor{yama-budo}{RGB}{171,14,122}

\title{New Statistical and Computational Results for Learning Junta Distributions}

\newtoggle{anonymous}
\togglefalse{anonymous}

\iftoggle{anonymous}{
\author{Anonymous Authors}
}{
\author[1]{Lorenzo Beretta}
\affil[1]{University of California, Santa Cruz, \texttt{lorenzo2beretta@gmail.com}}
}

\date{}

\begin{document}

\maketitle

\begin{abstract}
We study the problem of learning junta distributions on $\boolset^n$, where a distribution is a $k$-junta if its probability mass function depends on a subset of at most $k$ variables. We make two main contributions:

\begin{itemize}
\item We show that learning $k$-junta distributions is \emph{computationally} equivalent to learning $k$-parity functions with noise (LPN), a landmark problem in computational learning theory.

\item We design an algorithm for learning junta distributions whose statistical complexity is optimal, up to polylogarithmic factors.
Computationally, our algorithm matches the complexity of previous (non-sample-optimal) algorithms.
\end{itemize}
Combined, our two contributions imply that our algorithm cannot be significantly improved, statistically or computationally, barring a breakthrough for LPN. 
\end{abstract}

\section{Introduction}

We study the problem of learning $k$-junta distributions, as first introduced by Aliakbarpour, Blais, and Rubinfeld \cite{aliakbarpour2016learning}.
Given $n \in \mathbb N$ and $k \leq n$, a distribution $\cP$ supported on $\boolset^n$ is a $k$-junta distribution if its probability mass
function $\cP(x) = \mathbb P_{y \sim \cP}[y = x]$ 
is a $k$-junta function, 
where a $k$-junta function on $\boolset^n$ is a function that only depends on $k$ out of $n$ input variables\footnote{
More generally, \cite{aliakbarpour2016learning} defines a $k$-junta distribution with respect to a fixed distribution $\cD$. For $k \leq n$ and a fixed distribution $\cD$ supported on $\boolset^n$, a distribution $\cP$ over $\boolset^n$ is a $k$-junta distribution with respect to $\cD$ if there exists a size-$k$ set of coordinates $J \subseteq [n]$ such that, for every $x \in \boolset^k$, the distributions $\cP$ and $\cD$ conditioned on coordinates in $J$ being set according to $x$ are equal. When $\cD$ is the uniform distribution, the above definition is equivalent to the requirement that the probability mass function of $\cP$ is a $k$-junta function.
}.
The objective of the learning problem is to design algorithms that, given i.i.d. samples from an unknown $k$-junta distribution $\cP$ over $\boolset^n$, output the p.m.f. of a hypothesis distribution $\cD$ that satisfies $||\cD - \cP||_{TV} \leq \varepsilon$.

In this work, we show that learning junta distributions is \emph{computationally} as hard as learning parity functions with noise, a landmark problem in learning theory (\Cref{fig:reductions}).
We complement this result by designing a computationally-efficient algorithm for learning junta distributions that achieves almost optimal statistical complexity, improving over previous work (\Cref{tab:algorithmic results}).

\subsection{Related Work}

\paragraph*{Feature selection and learning juntas}
A key challenge in machine learning is learning target concepts despite the presence of irrelevant features \cite{chandrashekar2014survey}.
As observed in \cite{blum1997selection}, the task of PAC learning in the presence of irrelevant features can be formalized as the problem of learning the concept class of $k$-junta functions for $k \ll n$.  
Likewise, feature selection in the context of \emph{distribution learning} is formalized by the task of learning junta distributions.
Statistical and computational aspects of learning (and testing) juntas have been the subject of extensive work \cite{valiant2015finding, mossel2004learning, sauglam2018near, bshouty2018exact, servedio2015adaptivity, blais2008improved, arvind2009parameterized, arpe2007learning, chen2023testing, bao2023testing, kearns1998efficient, feldman2006new, blum2003noise}.

\paragraph*{Learning junta functions}
Computationally, learning junta functions is hard.
In the \emph{distribution-free} setting, learning $k$-juntas takes $n^{\Omega(k)}$ time, assuming randomized ETH \cite{koch2023superpolynomial}.
Interestingly, characterizing the complexity of learning juntas with respect to the \emph{uniform distribution} proved to be an elusive problem that resisted decades of attempts \cite{mossel2004learning,valiant2015finding, blum1994relevant, feldman2006new, blum2003open}. The fastest algorithm to date runs in time $\approx n^{0.6 k}$, which is only moderately faster than brute-force \cite{valiant2015finding, alman2018illuminating}.

\paragraph*{Learning parity functions with noise}

Learning a function $f : \boolset^n \rightarrow \boolset$ under distribution $\cD$ with noise rate $\eta < 1/2$ means recovering $f$ while observing pairs $(x, f(x) \cdot b)$, where $x \sim \cD$ and $b \in \boolset$ is a Rademacher random variable with parameter $\eta$. A function $f$ is a parity with relevant variables $J \subseteq [n]$ if $f(x) = \prod_{i \in J} x_i$. Learning parities with noise under the uniform distribution (LPN) is a cornerstone problem in computational learning theory, serving as a foundation for several hardness results (see \Cref{fig:reductions}).

Feldman, Gopalan, Khot, and Ponnuswami first demonstrated the central role of LPN by computationally reducing several fundamental problems in learning theory to it \cite{feldman2006new, feldman2009agnostic}. They showed that, under the uniform distribution, agnostically learning $k$-parity functions, learning $k$-junta functions (both with and without noise), and learning Boolean functions that admit a DNF representation of size $2^k$ all reduce to learning $k$-parities with noise.
Furthermore, LPN can be reduced to numerous other learning problems: learning halfspaces \cite{klivans2014embedding} and ReLU functions \cite{goel2019time} under the Gaussian distribution, learning submodular functions \cite{feldman2013representation} and agnostically learning halfspaces under the uniform distribution \cite{kalai2008agnostically}. 

LPN is widely believed to be computationally hard. Kearns established that LPN requires $\Omega(n^k)$ constant-noise queries in the statistical query (SQ) model \cite{kearns1998efficient, barak2022hidden}.
The learning with errors (LWE) problem, a generalization of LPN over $\mathbb{F}_q$, is widely conjectured to be hard and underpins modern lattice-based cryptography \cite{regev2009lattices}. Moreover, LPN itself has been leveraged as a hardness assumption in cryptographic constructions \cite{kiltz2017efficient}.

The best known algorithm for sparse LPN ($k \ll n$) runs in time $\approx n^{0.8k}$ \cite{valiant2015finding, alman2018illuminating}, and neural networks trained with SGD encounter the same $n^k$ computational barrier \cite{barak2022hidden}.
Finally, for $k \approx n$ the fastest known algorithm for LPN runs in time $2^{\frac{n}{\log n}}$, and extending this to $n^{\frac{k}{\log k}}$ for arbitrary $k$ remains a major open problem \cite{blum2003noise}.

\paragraph*{Learning junta distributions}
The problem of learning junta distributions (LJD) was introduced by Aliakbarpour, Blais, and Rubinfeld \cite{aliakbarpour2016learning}. 
In their work, they observed that the \emph{cover method}\footnote{See \cite{diakonikolas2016learning} for a general introduction to the method.} gives an algorithm for LJD with (essentially) optimal sample complexity (see \Cref{tab:algorithmic results} for exact  statistical and computational bounds).
Computationally, the cover-method algorithm suffers from a doubly-exponential dependence on $k$.
The main algorithmic contribution of \cite{aliakbarpour2016learning} is to design an algorithm that achieves low sample complexity ($\approx 2^{2k}\log n$) while being computationally efficient ($\approx n^k$).
While a computational cost of $n^k$ does not look efficient at first sight, \cite{aliakbarpour2016learning} proved that the problem of learning junta functions without noise (LJ) reduces to LJD, which rules out algorithms substantially faster than $n^k$, barring a breakthrough for LJ.

Chen, Jayaram, Levi and Waingarten studied the problem of learning junta distributions in a stronger model featuring \emph{subcube conditioning} \cite{chen2021learning}.
Subcube conditioning allows to sample $x$ conditioned on $x_i = q_i$ for all $i \in I$, given $I \subseteq [n]$ and $q\in \boolset^I$.
The power of this stronger model makes the problem computationally easier; in fact, they show an algorithm running exponentially faster than previous work. Surprisingly, this stronger model does not make the problem any easier statistically, as the lower bounds of \cite{chen2021learning} match the statistical complexity achieved by the cover method.  

Finally, Escudero Gutiérrez improved the algorithmic result of \cite{aliakbarpour2016learning} statistically, shaving a factor $2^k$ while keeping the computational cost at $\approx n^k$ \cite{escudero2024learning}.

\paragraph*{Truncated distributions}
Learning junta distributions connects with an emerging line of work on testing distributions truncated by low-degree polynomials.
A distribution is truncated by a function $f$ if it is obtained by conditioning a ground distribution $x\sim \cD$ on the value of $f(x)$.
Recently, testing junta truncation \cite{he2023testing}, as well as low-degree polynomial truncation \cite{de2024detecting}, have been considered.
Apparently, a uniform distribution truncated by a junta function is a junta distribution, so junta truncation testing reduces to LJD. 
Moreover, juntas are low-degree polynomials, so low-degree polynomial truncation is a generalization of junta truncation. 

\subsection{Our Results}
As stated above, \cite{aliakbarpour2016learning} reduced the problem of learning $k$-junta functions without noise (LJ) to that of learning $k$-junta distributions (LJD).
Since improving the computational complexity for LJ is a long-standing open problem, their reduction serves as a proof of \emph{conditional hardness} for LJD.
In this work, we strengthen this message and prove that LJD is computationally equivalent to learning $k$-parity functions with noise (LPN). See \Cref{fig:reductions} for a summary of these reductions.

\begin{restatable}{theorem}{MainTheoremComplexity}
\label[theorem]{thm:main thm complexity}
The computational complexity of these two problems is equal, up to $\poly(2^k, n)$ factors.
\begin{itemize}
    \item Learning $k$-parity functions over $\boolset^n$ with respect to the uniform distribution with noise rate $\eta = \frac 1 2 - 2^{-O(k)}$
    \item Learning $k$-junta distributions over $\boolset^n$ with sample access. 
\end{itemize}
\end{restatable}

Since LPN is a central problem in computational learning theory, we believe that \Cref{thm:main thm complexity} essentially closes the question about the computational complexity of LJD. 

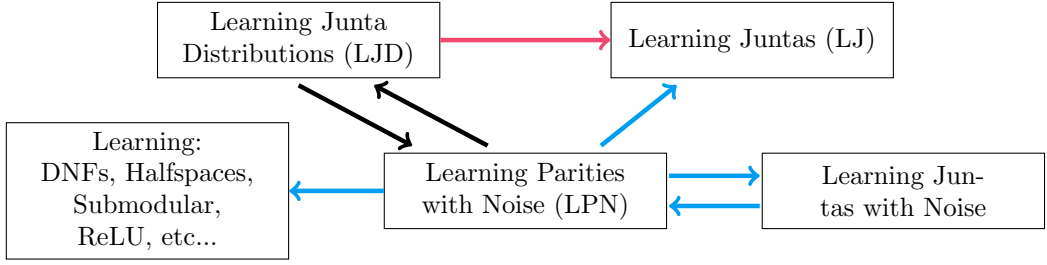
\begin{figure}
    \begin{tikzpicture}[
        node distance=2cm and 2.5cm,
        every node/.style={draw, text width=3.5cm, align=center, minimum height=1cm}
    ]
        \node (A) at (0,2) {Learning Junta Distributions (LJD)};
        \node (B) at (6,2) {Learning Juntas (LJ)};

        \node (C) at (-2,0) {Learning: \\  DNFs, Halfspaces, \\ Submodular, \\ ReLU, etc...};
        \node (D) at (3,0) {Learning Parities with Noise (LPN)};
        \node (E) at (8,0) {Learning Juntas with Noise};

        \draw[->, line width=.6mm] (0, 1.4) -- (1.5, .6);
        \draw[->, line width=.6mm,] (2.5, .6) -- (1, 1.4);
        \draw[->, line width=.6mm, momiji] (A) -- (B);
        \draw[->, line width=.6mm, ama-iro] (D) -- (C);
        \draw[->, line width=.6mm, ama-iro] (4, .6) -- (5, 1.4);
        \draw[->, line width=.6mm, ama-iro] (6.1, -.2) -- (4.9, -.2);
        \draw[->, line width=.6mm, ama-iro] (4.9, .2) -- (6.1, .2);

    \end{tikzpicture}
    \caption{An arrow from $A$ to $B$ means that problem $B$ reduces to problem $A$ (i.e., $A$ is harder than $B$). The blue arrows were proven in previous work,
    the red arrow was proven in \cite{aliakbarpour2016learning} and the black arrows are proven in this work.} \label{fig:reductions}
\end{figure}

On the algorithmic side, we design a new algorithm for LJD which improves over \cite{aliakbarpour2016learning} and \cite{escudero2024learning}.
In particular, \cite{escudero2024learning} achieves a sample complexity of $\approx 2^k \cdot \log n$ and we improve it to $\approx 2^k + \log n$, which is a quadratic speedup for $k = \log\log n$.  Moreover, our algorithm runs in time $\approx n^k$ which we expect to be hard to improve significantly, given \Cref{thm:main thm complexity}. 
Finally, the sample complexity achieved by our algorithm matches the lower bound of \cite{chen2021learning} up to polylogarithmic factors. 
See \Cref{tab:algorithmic results} for an exhaustive history of the computational and statistical complexity of learning junta distributions.

\begin{restatable}{theorem}{MainThmStatistical}
\label[theorem]{thm:main thm statistical}
There exists an algorithm that properly learns $k$-junta distributions up to TV distance $\varepsilon$ using $O(\frac{k}{\varepsilon^2} (2^k + \log n))$ samples and $O(\min(2^n, n^k) \cdot \frac{k}{\varepsilon^2} (2^k + \log n))$ running time.
\end{restatable}

\begin{table}[h]
    \centering
    \begin{tabular}{lcc}
        \toprule
         Algorithms & Samples & Time \\
        \midrule
         Cover Method &  $O\left( \frac{1}{\varepsilon^2} \cdot (2^k \log \frac{1}{\varepsilon} + k \log n) \right)$ & $O\left({\binom{n} {k}} \cdot \varepsilon^{-\Omega(2^k)} \right)$ \\
            \addlinespace[5pt]
        \cite{aliakbarpour2016learning} &  $O\left(\frac{k}{\varepsilon^4} \cdot 2^{2k} \log n \right)$ & $\Tilde O \left( \min(2^n, n^k) \cdot \frac{k}{\varepsilon^4} \cdot 2^{2k}\right)$ \\
            \addlinespace[5pt]
        \cite{escudero2024learning} &  $O\left(\frac{k}{\varepsilon^2} \cdot 2^{k} \log n \right)$  &  $O\left( \min(2^n, n^k) \cdot \frac{k}{\varepsilon^2} \cdot 2^{k} \log n \right)$ \\
            \addlinespace[5pt]
        \textbf{This Work} & $O\left(\frac{k}{\varepsilon^2} (2^k + \log n)\right)$  & $O\left(\min(2^n, n^k) \cdot \frac{k}{\varepsilon^2} (2^k + \log n)\right)$ \\
        \addlinespace[3pt]
        \midrule
        Lower Bounds & & \\
        \midrule
            \addlinespace[3pt]
         \cite{chen2021learning} &  $\Omega\left(\frac{1}{\varepsilon^2} \cdot 
         \left( 2^k + \log {\binom{n}{k}}\right) \right)$ &  \\
            \addlinespace[2pt]
        \bottomrule
    \end{tabular}
    \caption{Statistical and computational complexity of learning junta distributions.}
    \label{tab:algorithmic results}
\end{table}

\subsection{Technical Overview: Computational Complexity}
 
In order to prove \Cref{thm:main thm complexity}, we need to reduce learning parities with noise (LPN) to learning junta distributions (LJD) and vice versa.

\subsubsection{Reducing LPN to LJD}
Let $(x, y)$ be a sample-label input for LPN with ground truth $\chi_J(x) = \prod_{j \in J}x_j$. 
The sample $x \in \boolset^n$ is distributed uniformly, whereas $y \in \boolset$ is distributed as $\chi_J(x)$ flipped with probability $\eta < 1/2$. 

We implement a filtering strategy that, given LPN samples $(x, y)$, returns samples $z \sim \cD$, where $\cD$ is a junta distribution. Then, we show that learning $\cD$ suffices to recover $\chi_J$. 
We define $z \leftarrow x$ if $y = 1$ and re-sample $(x, y)$ otherwise.
This ensures that $\cD$ is a mixture of two distributions: (i) The uniform distribution supported on $\{w \in \boolset^n \,|\, \chi_J(w) = 1\}$; (ii) The uniform distribution $\cU_n$ over $\boolset^n$.
In this mixture (i) is weighted $\approx 1-2\eta$ and (ii) is weighted $\approx 2\eta$.
Moreover, $\cD$ is a $k$-junta distribution with relevant variables $J$. 

By running an exponential search, we can assume that we know $1 - 2\eta$ up to a constant factor.
We run our LJD learner on $\cD$ with precision $\varepsilon \cdot (1-2\eta)$, and it returns $\cF$ such that  $\lpnorm{1}{\cF - \cD} \leq\varepsilon \cdot (1-2\eta)$.
Now, define a boolean function $f: \boolset^n \rightarrow \boolset$
as $f(x) = 1$ iff $\cF(x) > 1/2$. 
Our estimate $\cF$ is precise enough so that $f$ satisfies $\Exui{x}{|f - \chi_J|} \leq \varepsilon$.
Notice that since we have query access to the probability mass function of $\cF$, we also have query access to $f$.
In order to retrieve $J$ from $f$, consider $f(w) \cdot f(w \oplus e_j)$ where $w \in \boolset^n$ is uniformly distributed and $w \oplus e_j$ flips the $j$-th bit of $w$. By virtue of $\Exui{x}{|f - \chi_J|} \leq \varepsilon$, we must have that $f(w) \cdot f(w \oplus e_j)$ has a majority of $-1$ for $j \in J$ and a majority of $1$ for $j \not \in J$. To conclude, we just test if $j \in J$ for all $j \in [n]$.

\subsubsection{Reducing LJD to LPN}
In order to reduce LJD to LPN we need to introduce an auxiliary problem: \emph{learning noisy parity distributions} (LNPD).
We say that a junta distribution is a parity distribution with relevant variables $J$ if it coincides with the uniform distribution over $\{w\in \boolset^n\,|\, \chi_J(w) = s\}$ for a fixed $s \in \boolset$.
We generalize this definition and say that a junta distribution is a noisy parity distribution if it is a mixture of a uniform distribution and a parity distribution.
By symmetry, the only non-zero coefficients in the Fourier spectrum of (the p.m.f. of) a noisy parity distribution $\cD$ with relevant variables $J$ 
are $\hD(\emptyset)$ and $\hD(J)$.
The reduction of LJD to LPN develops in two steps: First, we reduce LJD to LNPD; Second, we reduce LNPD to LPN.

\paragraph*{Reducing LJD to LNPD}
The Fourier spectrum of a junta distribution can have as many as $2^k$ non-zero coefficients, whereas a noisy parity distribution has at most two. The idea is to distill our input junta distribution $\cD$ to $\cP$ so that $\cP$ is a noisy parity with relevant variables $J$ such that $\hP(J) = \hD(J)$, for some fixed $J$.
Then, use the LNPD algorithm to learn $\cP$. 

To perform such distillation, we use an idea first introduced in \cite{feldman2006new}:
By carefully adding noise to the distribution $\cD$ it is possible to zero all Fourier coefficients of $\cD$ besides $\hD(J)$ and $\hD(\emptyset)$. Moreover, each $J \subseteq [n]$ with $\hD(J) \neq 0$ has a large enough chance of being left intact. Thus, a coupon-collector style argument ensures that we learn all non-zero Fourier coefficients of $\cD$ in $\approx k \cdot 2^k$ rounds.

\paragraph*{Reducing LNPD to LPN}
For ease of presentation, we show how to reduce learning (non-noisy) parity distributions to learning parity functions (without noise). The reduction carries over to the noisy case.
Let $\cD$ be a parity distribution over relevant variables $J$ and such that $\chi_J(w) = 1$ for all $w \in \supp(\cD)$ (the case $\chi_J(w) = -1$ is symmetric). Let $z \sim \cD$. 
Fix $j \in [n]$ and define $(x, y)$ as follows: Flip the $j$-th bit of $z$ with probability $1/2$ and define that to be $x$; define $y = -1$ if the $j$-th bit was flipped and $y = 1$ otherwise. 
If $j \in J$, then $x$ is distributed uniformly over $\boolset^n$ and $y = \chi_{J}(x)$. Else, if $j \not \in J$ then $y$ is distributed uniformly on $\boolset$ and independently of $x$. 

Then, we solve LPN on samples $(x, y)$.
If it turns out that $y$ is a $k$-parity function of $x$ then we apply the same technique used in the reduction from LPN to LJD above to retrieve $S$ from $f$, and we find $J$. 
If $j \not\in J$ then $y$ is not going to be a parity function of $x$ and we need to redefine $\cD$ using some other $j \in [n]$. By scanning all $j \in [n]$, we are guaranteed that at some point $j \in J$ holds.

\subsection{Technical Overview: Sample-Optimal Algorithm}

Our algorithm learns JDN with fewer samples than previous work; in particular, it replaces the term $2^k \log n$ in the sample complexity of \cite{escudero2024learning} with $2^k + \log n$ (see \Cref{tab:algorithmic results}).
The goal of this overview is to show how to make the term $\log n$ additive, rather than multiplicative. 

\paragraph*{Low-degree algorithms}
Most previous literature on junta distributions, as well as our own algorithm, make use of the low-degree algorithm of Linial, Mansour, and Nisan \cite{linial1993constant}.
In a nutshell, the low-degree algorithm estimates individual Fourier coefficients and outputs a hypothesis function consistent with the estimated Fourier coefficients. 
In \cite{aliakbarpour2016learning}, the authors pointed out that applying the low-degree algorithm to the junta distribution problem requires too many samples. Indeed, an accuracy of $\varepsilon \cdot 2^{-n/2}$ on individual Fourier coefficients is necessary to deliver an $\varepsilon$ error in $L_2$-norm.
To bypass this, Escudero introduced a thresholding of Fourier coefficients, thus ensuring that at most $2^k$ coefficients are estimated to be nonzero \cite{escudero2024learning}.
In this work, we simply apply the low-degree algorithm to distributions over a small domain $\boolset^k$, thus the standard analysis from \cite{linial1993constant} suffices.

\paragraph*{Our template algorithm}
Consider the following template.
We maintain a junta distribution $\cQ$ as our ``best guess'' for $\cD$, along with its set of relevant variables $N \subseteq [n]$ of size at most $k$.
We maintain the invariant that $N \subseteq J$, where $J$ is the true set of relevant variables. 
Given $S \supseteq N $, we can check if the marginals of $\cQ$ and $\cD$ with respect to the variables in $S$ (resp. $\cQ|_S$ and $\cD|_S$) are close in total variation distance.
In order to check the above condition, we draw samples from $\cD$, project them onto the coordinates in $S$ and feed these samples to a closeness tester.

If $\lpnorm{TV}{\cQ|_S - \cD|_S} \leq \varepsilon$ for all $S \in \binom{[n]}{k}$ with $N \subseteq S$, then we can correctly output $\cQ$. 
Else, if we find  $S \in \binom{[n]}{k}$ with $N \subseteq S$ such that $\lpnorm{TV}{\cQ|_S - \cD|_S} > \varepsilon$, then we learn the Fourier coefficients of $\cD|_S$ using the low-degree algorithm.
Then, we update the distribution $\cQ$ so that it includes the newly recovered Fourier coefficients, and augment $N$ accordingly.
Each iteration of this scheme increases the size of $N$, so we are done after $k$ iterations.

\paragraph*{Sample complexity analysis}
First, observe that while making multiple calls to an algorithm we can certainly re-use the same samples, and the price we pay is that we need to union bound over the error probability.
The low-degree algorithm operates on distributions over $\boolset^k$, so the cost of an individual call is negligible. Moreover, we make at most $k$ calls to it, so a standard boosting of the success probability increases its sample complexity by at most $\log k$.
On the other hand, we make up to $k \cdot n^k$ calls to the closeness tester. Each individual call costs at least $\tilde \Omega(2^k)$ by known lower bounds \cite{valiant2017estimating}. So, we cannot afford a standard boosting of the success probability, as it would introduce a multiplicative term $\log (n^k) = k \log n$.

To counter this, we develop a closeness tester that succeeds with probability $1 - \delta$ and uses only $s + \log 1/ \delta$ samples, where $s$ is the support size.
The design of our tester uses the TV distance between the empirical probability mass function and the tested distribution as an estimator, borrowing ideas from \cite{diakonikolas2018sample}.

\subsection{Organization}
In \Cref{sec:preliminaries} we introduce notation. In \Cref{sec:complexity} we prove the computational equivalence of learning junta distributions and learning parities with noise (\Cref{thm:main thm complexity}). In \Cref{sec:algorithm} we show our new algorithm for the problem of learning junta distributions (\Cref{thm:main thm statistical}).
\section{Preliminaries}
\label{sec:preliminaries}
We denote by $\Delta(A)$ the set of probability distributions over the set $A$.
Given a probability distribution $\cP \in \Delta(\boolset^n)$ and a set $S \subseteq [n]$ we denote by $\cP|_S$ the marginal of $\cP$ with respect to variables in $S$.
We overload the notation and denote by $\cP(\cdot)$ its probability mass function (PMF).
We denote the uniform distribution over $\boolset^n$ by $\cU_n$.
Whenever not specified, expectations are taken with respect to the uniform distribution. 

\paragraph*{Fourier transform}
In this paper we use the Fourier transform over the boolean hypercube. 
For a subset $A \subseteq [n]$ we define the function $\chi_A:\boolset^n \rightarrow \boolset$ such that $\chi_A(x) = \prod_{i \in A} x_i$.
Such functions form the Fourier basis of the vector space of functions $\boolset^n \rightarrow \bbR$. Notice that the Fourier basis is orthonormal with respect to the dot product $\langle f, g\rangle = \Exui{x}{f(x) \cdot g(x)}$.

Given a function $f: \boolset^n \rightarrow \bbR$ we define its Fourier transform $\hat f: 2^{[n]} \rightarrow \bbR$ as $\hat f(A) = \Exui{x}{f(x) \cdot \chi_A(x)}$. Thus, $f = \sum_{A \subseteq [n]} \hat f (A) \cdot \chi_A$.
Finally, we have
\[
\Exu{x\sim \cU_n}{f^2(x)} = 2^{-n} \cdot\lpnorm{2}{f}^2 = \sum_{A \subseteq [n]} \hat f(A)^2.
\]

\paragraph*{Learning functions and distributions}
Fix a class of boolean functions $F \subseteq \{f: \boolset^n \rightarrow \boolset\}$.
Consider an algorithm $\cA$ that takes as input parameters $\varepsilon, \delta > 0$ and  $\{(x_i, f(x_i))\}_{i \in [m]}$ where $x_1 \dots x_m \in \boolset^n$ are i.i.d. samples from a distribution $\cD$ and $f \in F$ is fixed. Moreover, $\cA$ outputs a function $f_\cA$ that satisfies $\Prui{x\sim \cD}{f(x) \neq f_\cA(x)} \leq \varepsilon$ with probability at least $1-\delta$. Then, we say that $\cA$ learns $F$ using $S(n)$ samples and $T(n)$ time under the distribution $\cD$ if $\cA$ takes $S(n) \cdot \poly(\varepsilon^{-1}, \delta^{-1})$ samples and runs in time $T(n) \cdot \poly(\varepsilon^{-1}, \delta^{-1})$.
Moreover, we say that $\cA$ learns $F$ with noise of rate $\eta < 1/2$ if the label $f(x_i)$ in the samples $(x_i, f(x_i))$ is flipped with probability $\eta$.

Likewise, fix a class of distributions $P \subseteq \Delta(\boolset^n)$.
Consider an algorithm $\cA$ that takes as input parameters $\varepsilon, \delta > 0$ and $x_1 \dots x_m$, where $x_1 \dots x_m \in \boolset^n$ are i.i.d. samples from a fixed distribution $\cP \in P$. Furthermore, $\cA$ outputs a distribution $\cP_\cA$ that satisfies $\lpnorm{TV}{\cP - \cP_\cA(x)} \leq \varepsilon$ with probability at least $1-\delta$. Then, we say that $\cA$ learns $P$ using $S(n)$ samples and $T(n)$ time if $\cA$ takes $S(n) \cdot \poly(\varepsilon^{-1}, \delta^{-1})$ samples and runs in time $T(n) \cdot \poly(\varepsilon^{-1}, \delta^{-1})$.

In general, we can require an algorithm that learns a distribution to output either a \emph{generator} or an \emph{evaluator} for the learned distribution \cite{kearns1994learnability}.
A generator is an algorithm that generate samples from the (approximate) target distribution, whereas an evaluator returns the value of the probability mass function on a given query point. In this work, we assume that a distribution learner returns evaluators.

Before delving into the proofs of \Cref{thm:main thm complexity} and \Cref{thm:main thm statistical}, we prove two folklore lemmas that are going to be useful throughout the paper.

\begin{lemma}
\label[lemma]{lem:k-junta closeness through projections}
Let $\cP, \cQ \in \Delta(\boolset^n)$ be junta distributions with set of relevant variables $J \subseteq [n]$, then
\[
||\cP - \cQ||_1 = ||\cP|_{J} - \cQ|_{J}||_1.
\]
\end{lemma}
\begin{proof}
For each $x \in \boolset^{J}$, define $S_x = \ld\{y \in \boolset^n\,|\,  y_j = x_j \text{ for all } j \in J \rd\}$ and notice that $\cP$ and $\cQ$ are both constant on $S_x$.
\begin{align*}
||\cP - \cQ||_1 &= \\
\sum_{x \in \boolset^{J}} \sum_{y \in S_x} |\cP(y) - \cQ(y)| &= \\
\sum_{x \in \boolset^{J}} \ld|\sum_{y \in S_x} \cP(y) - \cQ(y)\rd| &= \\
\sum_{x \in \boolset^{J}} \ld|\cP|_{J}(x) - \cQ|_{J}(x)\rd| &= \\
||\cP|_{J} - \cQ|_{J}||_1.
\end{align*}
\end{proof}

\begin{lemma}
\label[lemma]{lem:fourier coefficients of marginals}
Given a distribution $\cP$ over $\boolset^n$ and a size-$k$ set $S \subseteq [n]$ we have that for each $J \subseteq S$ 
\[
\hat{\cP|_S}(J) = 2^{n-k} \cdot \hP(J).
\]
\end{lemma}
\begin{proof}
For any distribution $\cD$ over $\boolset^n$ we have
\begin{align*}
\hat{\cD}(J) = 2^{-n} \cdot \left(\Pru{x\sim \cD}{\Pi_{i\in J} x_i = 1} - \Pru{x\sim \cD}{\Pi_{i\in J} x_i = -1} \right).
\end{align*}
Applying this identity to both sides proves the desired result.
\end{proof}
\section{The Complexity of Learning Junta Distributions}
\label{sec:complexity}

In this section we prove the following theorem.

\MainTheoremComplexity*

In order to prove \Cref{thm:main thm complexity}, we show mutual reductions between learning junta distributions (LJD) to learning parity functions with noise (LPN).
In order to reduce LJD to LPN we need to introduce an auxiliary problem: \emph{learning parity distributions with noise} (LPDN). 
In \Cref{sec:reducing noisy parity distributions to parity functions} we reduce LPDN to LPN (see \Cref{lem:reduction from noisy parity distribution to junta functions}).
In \Cref{sec:reducing junta distribtions to noisy parity distributions} we reduce LJD to LPDN (see \Cref{lem:reduction from junta distribution to parity distribution}).
Finally, in \Cref{sec:reducing noisy parity function to junta distributions} we reduce LPN to LJD (see \Cref{lem:reduction from noisy parities to junta distributions}).

  \begin{figure}
\center
    \begin{tikzpicture}[
        node distance=2cm and 2.5cm,
        every node/.style={draw, text width=2cm, align=center, minimum height=1cm}
    ]
        \node (A) at (0,2) {LPN};
        \node (B) at (10,2) {LJD};
        \node (C) at (5,2) {LPDN};

        \draw[->, line width=.6mm] (B) to[out=165, in=15] (A);
        \draw[->, line width=.6mm] (A) to (C);
        \draw[->, line width=.6mm] (C) to (B);
        
    \end{tikzpicture}
    \caption{An arrow from $A$ to $B$ means that we reduce problem $B$ to problem $A$.} 
    \label{fig:intermediate reductions}
\end{figure}
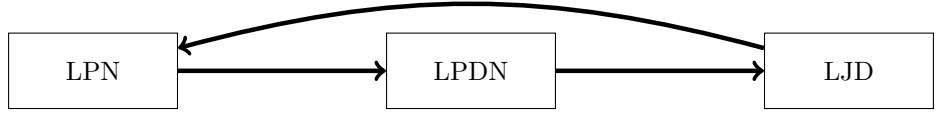

\subsection{Reducing Noisy Parity Distributions to Parity Functions with Noise}
\label{sec:reducing noisy parity distributions to parity functions}

First, we define noisy parity distributions.

\begin{definition}[Noisy parity distribution]
\label[definition]{def:noisy parity distribution}
We say that a distribution $\cD$ over $\boolset^n$ is a noisy parity distribution with relevant variables $J \subseteq [n]$ with noise rate $\eta < 1/2$ if 
\begin{equation}
\label{eq:def of noisy parity}
\hD(S) = \frac 1 {2^n} \cdot \begin{cases}
		1 &\text{if $S = \emptyset$} \\
        \pm(1 - 2\eta) &\text{if $S = J$} \\
        0 &\text{otherwise.}
	\end{cases}
\end{equation}
\end{definition}

\begin{observation}
\label[observation]{obs:noisy parity distribution as mixture}
A noisy parity distribution with noise rate $\eta = 0$ corresponds to the uniform distribution truncated by a parity function (i.e., $x \sim \cU_n$ conditioned on either $\chi_J(x) = \pm 1$). We call such distribution a \emph{parity distribution} and highlight that the PMF of a parity distribution is, indeed, a parity function (up to multiplying it by $-1$).
Moreover, a noisy parity distribution over set $J$ with noise rate $\eta > 0$ coincides with the mixture $(1-2\eta) \cdot \cP + 2\eta \cdot \cU_n$, where $\cP$ is a parity distribution. 
\end{observation}

The bulk of this section is spent proving the next lemma, which reduces learning noisy parity distributions to learning parity functions with noise.

\begin{lemma}
\label[lemma]{lem:reduction from noisy parity distribution to junta functions}
Let $\cA$ be an algorithm that  learns 
$k$-parity functions
over $\boolset^n$ under the uniform distribution with noise of rate $\eta < 1/2$ and uses $T\ld(n, k, (1-2\eta)^{-1}\rd)$ time.
Then, there exists an algorithm $\cA'$ that  learns noisy $k$-parity distributions over $\boolset^n$ with noise rate $\eta < 1/2$. Moreover, $\cA'$ uses at most $\poly(2^k, n) \cdot T\ld(n, k, (1-2\eta)^{-1}\rd)$ time.
\end{lemma}

Before delving into the proof of \Cref{lem:reduction from noisy parity distribution to junta functions}, we prove the following lemma.
\begin{lemma}
\label[lemma]{lem:learning one fourier coefficient}
Fix $\varepsilon > 0$. 
Let $\cD$ be a noisy parity distribution over $\boolset^n$ with relevant variables $J \subseteq [n]$ and noise rate $\eta < 1/2$. 
Then, there exists an algorithm that takes as input $s$ candidates sets $J_1 \dots J_s$ promised to satisfy $J \in \{J_1 \dots J_s\}$ and, without knowledge of $\eta$, estimates each $\hat \cD(J_i)$ up additive error $\pm \varepsilon \cdot \hD(J)$ with probability $1 - \delta$. 
Moreover, such algorithm finds $J$, and uses at most $O(s\cdot \varepsilon^{-2} (1-2\eta)^{-2} \log \delta^{-1})$ samples from $\cD$ and time. 
\end{lemma}
\begin{proof}
By \Cref{def:noisy parity distribution}, we have $\hD(J) \in\{\pm2^{-n} \cdot (1-2\eta)\}$. Assume $\hD(J) > 0$ without loss of generality.
We have $\Exui{x \sim \cD}{ 2^{-n} \cdot \chi_{J_i}(x)} =  \hat \cD(J_i) = 2^{-n} \cdot (1-2\eta) \cdot [J_i = J]$ and $\Exui{x \sim \cD}{ (2^{-n}\cdot \chi_{J_i}(x))^2} = 2^{-2n}$.
Let
$\eta \leq \bar \eta < 1/2$, $m = O(s\cdot \varepsilon^{-2} (1-2\bar\eta)^{-2} \log \delta^{-1})$ and given i.i.d. samples $x_1 \dots x_m \sim \cD$ define $\hat y_i= \frac{1}{m} \cdot (\chi_{J_i}(x_1) + \dots + \chi_{J_i}(x_m))$ and $\tilde y_i =2^{-n} \cdot \hat y_i$.
Standard concentration bounds imply that $\tilde y_i$ approximates $\hD(J_i)$ up to an additive error $\pm \varepsilon \cdot (1-2\eta) \cdot 2^{-n}$ with probability at least $1 - \delta \cdot 2^{-s}$.

Since $\eta$ is not known, we perform an exponential search over $\bar\eta$ by shrinking $1/2 - \bar \eta$ by half at each iteration. 
Define $Y = \{i \in [s] \,|\, J_i = J\}$ and $N = [s] \setminus Y$ (recall that $Y$ and $N$ are unknown to the algorithm). 
At each step, we have that, for $i\in Y$, $\hat y_i = 1-2\eta \pm \varepsilon (1-2\bar \eta)$, and, for $i \in N$, $\hat y_i = \pm \varepsilon (1-2\bar \eta)$ with high probability. Condition on this event. Then, we can halt whenever there exists an $i \in [s]$ such that $||\hat y_i| - \varepsilon\cdot(1-2\bar\eta)| \geq (1-2\bar\eta)$ as this guarantees $\bar\eta \geq \eta$ with high probability. Notice that we halt as soon as $1-2\bar\eta < \frac{1}{2} \cdot (1-2\eta)$, which gives the desired complexity.

Once we have successfully approximated $\hD(J_i)$ for all $i \in [s]$ up to an additive term $\pm \varepsilon \cdot (1-2\eta) \cdot 2^{-n} = \pm \varepsilon \cdot \hD(J)$, we must have $J = J_i$ for all $J_i$ satisfying $\tilde y_i > \varepsilon (1-2\bar\eta) \cdot 2^{-n}$.
\end{proof}

Finally, we prove \Cref{lem:reduction from noisy parity distribution to junta functions}.

\begin{proof}[Proof of \Cref{lem:reduction from noisy parity distribution to junta functions}]

Let $\cD$ be the noisy parity distribution over the set $J \subseteq [n]$ with noise rate $\eta < 1/2$ that we would like to learn. 
By definition, $\cD$ has only two non-zero Fourier coefficients: $\hat \cD(\emptyset)$ and $\hat \cD (J)$. 
We can assume $\hD(J) > 0$, as the case $\hD(J) < 0$ is symmetric.
Since $\cD$ is a probability distribution, we know that $\hat \cD(\emptyset) = 2^{-n}$. We are left to learn $\hat \cD (J)$.
Once we find the set of relevant variables $J$, estimating $\hat \cD(J)$ is straightforward.

In order to find $J$, we proceed as follows. Suppose that we correctly guessed one element $j \in J$. Sample $x' \sim \cD$, define $x \leftarrow x'$ and flip the $j$-th bit of $x$ with probability $1/2$. 
Then, define the label $y$ as $y = -1$ if we flipped $x_j$ and $y= 1$ otherwise.
Let $\cQ$ be the distribution of $(x, y)$ generated as above and let $\cQ_x$ and $\cQ_y$ be its marginals. We will prove that
\begin{enumerate}[(i)]
    \item $\cQ_x$ is the uniform distribution over $\boolset^n$
    \item Conditioned on $x$, $y$ is distributed as $\chi_J(x)$ with noise $\eta$.
\end{enumerate}

To prove point (i), we compute a generic coefficient $\hQ_x(A)$. denote by $e_j \in \boolset^n$ the $j$-th vector of the standard basis.
\[
2^n \cdot \hQ_x(A) = \sum_{z \in \boolset^n} \cQ_x(z) \cdot \chi_A(z) = \sum_{z \in \boolset^n} \cD(z) \cdot \frac{1}{2} \ld(\chi_A(z) + \chi_A(z \oplus e_j)\rd)
\]
then, for $j \in A$ we have $\hQ_x(A) = 0$, whereas for $j \not\in A$ we have $\hQ_x(A) = \hD(A)$. Therefore, assuming $j \in J$, $\hQ_x(A) = 0$ for all $A \neq \emptyset$.

To prove (ii) we note that, by \Cref{obs:noisy parity distribution as mixture}, $\cD$ can be written as the mixture $(1-2\eta) \cdot \cP + 2\eta \cdot \cU_n$ where $\cP$ is the parity distribution obtained as the uniform distribution over $\boolset^n$ conditioned on $\chi_J(\cdot) = 1$.
Equivalently, there exists an event $\cE$ of probability $2\eta$ such that $\cD$ conditioned on $\cE$ equals $\cU_n$ and $\cD$ conditioned on $\neg \cE$ equals $\cP$. 
If we condition on $\cE$, $y$ is independent of $x$ and uniform over $\boolset$. If we condition on $\neg \cE$, we have $y = \chi_J(x)$. Therefore, $y$ is distributed as $\chi_J(x)$ with noise $\eta$. 

Now, we feed our samples $(x, y) \sim \cQ$ to the algorithm $\cA$, which learns the $k$-parity $\chi_J$ in the presence of noise of rate $\eta$. 
By paying a $\poly(\varepsilon^{-1}, n)$ overhead, we can ensure that $\cA$ learns $\chi_J$ up to $\ell_1$-accuracy $\varepsilon \cdot 2^n > 0$ with probability $1 - \poly(n^{-1})$.
Recall that $\cQ$ depends on our guess $j \in [n]$.
We repeat the procedure described above for all $j \in [n]$ and let $(g^j)_{j\in [n]}$ be the functions returned by $\cA$. Let $\cE$ be the event that, for all $j \in J$, $\cA$ is correct. $\cE$ happens with probability at least $1 - \poly(n^{-1})$. Condition on $\cE$. Then, for each $j \in J$ we have $||g^j - \chi_J ||_1 \leq \varepsilon \cdot 2^n$.

Now we show how to define a set of candidates $\{J_i\}_{i \in [n]}$ for $J$.
We need a procedure $\findcandidate(g)$ that returns a size-$k$ subset of $[n]$ and such that $\findcandidate(g)$ returns $J$ whenever $\lpnorm{1}{g-\chi_J} \leq \varepsilon \cdot 2^n$.
Then, we generate our set of candidates $\{J_i\}_{i \in [n]}$ as $J_i = \findcandidate(g^i)$. \Cref{lem:learn with queries} describes how to implement $\findcandidate$.
Finally, $(J_i)_{i \in [n]}$ is a collection of at most $n$ candidate sets and we are promised that $J \in (J_i)_{i \in [n]}$, so we run the algorithm from \Cref{lem:learning one fourier coefficient}.

\begin{lemma}[Find Parity with Queries] \label[lemma]{lem:learn with queries}
Assume that we have query access to a function $g : \boolset^n \rightarrow \boolset$ such that $\lpnorm{1}{g - \chi_S} \leq \varepsilon \cdot 2^n$ for some $S \subseteq [n]$.
Then, we can find $S$ using $\poly(n)$ queries and time.
\end{lemma}
\begin{proof}
We construct our candidate $I$ for $S$ one element at a time. Initialize $I \leftarrow \emptyset$.
We iterate over $\ell \in [n]$, and for each $\ell$ we sample $\poly(n)$ elements $w \sim \cU_n$ and observe the distribution of $g(w) \cdot g(w\oplus e_\ell)$. If the majority of those is $-1$, then add $\ell$ to $I$; otherwise continue with the next $\ell \in [n]$.

Assuming $\lpnorm{1}{g - \chi_S}\leq \varepsilon \cdot 2^n$, we have $\Prui{w\sim \cU_n}{g(w) \neq \chi_S(w) \text{ or } g(w \oplus e_\ell) \neq \chi_S(w \oplus e_\ell)} \leq 2\varepsilon$.
Thus, by standard concentration bounds we have that the majority of $g(w) \cdot g(w\oplus e_\ell)$ is $-1$ if and only if $\ell \in J$ with high probability. Conditioned on this event holding for each $\ell \in [n]$, we have $I \subseteq S$ and after $\ell$ rounds we have $S \cap [\ell] \subseteq I$.
The described algorithm uses $\poly(n)$ and queries and, with high probability, returns $I = S$. 
\end{proof}

\end{proof}

\subsection{Reducing Junta Distributions to Noisy Parity Distributions}
\label{sec:reducing junta distribtions to noisy parity distributions}

In this section, we reduce the problem of learning junta distributions to that of learning noisy parity distributions. The main result of this section is the next lemma.

\begin{restatable}{lemma}{ReduceJuntasToParities}
\label[lemma]{lem:reduction from junta distribution to parity distribution}
Let $\cA$ be an algorithm that  learns a noisy $k$-parity distribution over $\boolset^n$ with noise rate $\eta < 1/2$. Suppose that $\cA$ uses $T\ld(n, k, (1-2\eta)^{-1}\rd)$ time.
Then, there exists an algorithm $\cA'$ that  learns a $k$-junta distribution over $\boolset^n$. Moreover, $\cA'$ uses at most $\poly(2^k, n) \cdot T\ld(n, k, O(2^{k/2})\rd)$ time. 
\end{restatable}

To prove \Cref{lem:reduction from junta distribution to parity distribution}, we mimic the techniques used in \cite{feldman2006new} to reduce learning junta functions with noise to learning parity functions with noise.
Here is a summary of this technique.
Let $\cD$ be the junta distribution with relevant variables $J^\star$ that we would like to learn. We introduce some noise that, with large enough probability, zeros all Fourier coefficient of $\cD$ besides $\hD(\emptyset)$ and $\hD(J)$ for some $J \subseteq J^\star$.
Once all coefficients besides these two are zero, we are left exactly with a noisy parity distribution.
As long as we manage to isolate all $J \subseteq J^\star$ such that $\hD(J)$ is high enough, we can learn $\cD$.

In order to implement this scheme we need the following lemmas that implement noise injection.
Throughout the section we identify $\zerooneset^n \sim 2^{[n]}$ so that Fourier coefficients are (also) indexed by boolean vectors. For $J \subseteq [n]$, denote by $\one_J\in\zerooneset^n$ the vector such that $\one_J[i] = 1$ iff $i \in J$. 
We denote by $\oplus$ the sum modulo $2$.
We overload the notation $\cU_n$ so that it represent the uniform distribution both on $\boolset^n$ and $\zerooneset^n$.

\begin{lemma}[Essentially Lemma 1  in the full version of \cite{feldman2006new}]
\label[lemma]{lem:characterization of function with added noise}
Let $A \in \zerooneset^{m \times n}$ and let $f : \zerooneset^n \rightarrow \bbR$ be any function over the boolean hypercube. Then, for each $x \in \zerooneset^n$ we have the identity 
\[
f_\cA(x) := \Exu{p \sim \cU_n}{f (x \oplus A^\top p)} = \sum_{\substack{a \in \zerooneset^n \\ Aa = 0^m}} \hat f(a) \chi_a(x).
\]
\end{lemma}
\begin{proof}
First, notice that for each $a\in \zerooneset^n$ 
\[
\chi_a(A^\top p) = (-1)^{\langle a, A^\top p\rangle} =(-1)^{\langle A a, p\rangle}, 
\]
thus $\Exui{p\sim \cU_m}{\chi_a(A^\top p)} = 1$ if $A a = 0$ and $\Exui{p\sim \cU_m}{\chi_a(A^\top p)} = 0$ otherwise.
\begin{align*}
\Exu{p \sim \cU_n}{f (x \oplus A^\top p)} &= \\
\sum_{a \in \zerooneset^n} \hat f(a) \cdot \Exu{p\sim \cU_m}{\chi_a\ld(x \oplus A^\top p\rd)} &= \\
\sum_{a \in \zerooneset^n} \hat f(a) \cdot \chi_a(x) \cdot \Exu{p\sim \cU_m}{\chi_a\ld(A^\top p\rd)} &= \\
\sum_{\substack{a \in \zerooneset^n \\ Aa = 0^m}} \hat f(a) \chi_a(x).
\end{align*}
\end{proof}

\begin{lemma} \label[lemma]{lem:only one fourier survives}
Let $\cD$ be a $k$-junta distribution with relevant variables $J^\star \subseteq [n]$ and $m > k$.
Sample $A \in \zerooneset^{m \times n}$ uniformly at random.
Define $\cD_A(x) := \E_{p \sim \cU_n}[\cD (x \oplus A^\top p)]$ as in \Cref{lem:characterization of function with added noise}.
Fix $J \subset J^\star$. Then, with probability at least $2^{-(m+1)}$ we have
\begin{equation} \label{eq:desiderata only one fourier survives}
\hD_A(I) = \begin{cases}
			\hD(J), & \text{if $I = J$}\\
           2^{-n}, & \text{if $I = \emptyset$} \\
            0, & \text{otherwise.}
		 \end{cases}
\end{equation}
Namely, $\cD_A$ is a noisy parity distribution with relevant variables $J$.
\end{lemma}
\begin{proof}
By \Cref{lem:characterization of function with added noise}, we have
\[
\cD_\cA(x) := \Exu{p \sim \cU_n}{\cD (x \oplus A^\top p)} = \sum_{\substack{a \in \zerooneset^n \\ Aa = 0^m}} \hat \cD(a) \chi_a(x).
\]
To realize \Cref{eq:desiderata only one fourier survives}, we must have $A \one_J = 0^m$ and $A \one_{I} \neq 0^m$ for all $I \subseteq J^\star$ with $I \neq J$. Let $\cE$ be such event.
It is straightforward to verify that, for any $u, v \in \zerooneset^m$ and any $M, N \subseteq [n]$ distinct, we have 
\[
\Pr{A \cdot \one_M = u \,|\, A \cdot \one_N = v}  = 2^{-m}.
\]
Thus, the probability of event $\cE$ can be bounded as
\begin{align*}
\Pr{\cE} &= \\
\Pr{A \cdot \one_I \neq 0^m \text{ for all $I \subseteq J^\star$ with $I \neq J$} \,|\,  A \cdot \one_J = 0^m} \cdot \Pr{A \cdot \one_J = 0^m} &\geq \\
(1 - 2^k \cdot 2^{-m}) \cdot 2^{-m} &\geq 2^{-(m+1)}.
\end{align*}
\end{proof}

\begin{lemma}
\label[lemma]{lem:adding linear noise}
Given a probability mass function $\cD$ over $\zerooneset^n$, and a matrix $A \in \zerooneset_2^{m \times n}$ the function $\cD_A(x) := \E_{p \sim \cU_m}[\cD(x \oplus A^\top p)]$ is a probability mass function. Moreover,  we can simulate a sample from $\cD_A$ by $y + A^\top q \sim \cD_A$ where $y \sim \cD$ and $q \sim \cU_m$.
\end{lemma}
\begin{proof}
For any fixed $p \in \zerooneset^m$ we have that $x\mapsto \cD(x \oplus A^\top p)$ is a probability mass function (PMF), moreover the set of PMFs is convex so $\cD_A$ is also PMF.
We have 
\[
\Pru{y, z}{y \oplus A^\top q = z} = 
\Pru{y, z}{y = z \oplus A^\top q} = 
\Exu{q \sim \cU_m}{\cD(z \oplus A^\top q)} =
\cD_A(z).
\]
\end{proof}

\begin{lemma}
\label[lemma]{lem:tv distance Fourier bound}
Given two distributions $\cD, \cP$ over $\boolset^n$ we have 
\[
||\cD - \cP||_1 \leq 2^n \cdot \sqrt{\sum_{a \in \zerooneset^n} \ld(\hD(a) - \hP(a)\rd)^2}.
\]
Moreover, if $\cD$ and $\cP$ are $k$-junta distributions over the same set of relevant variables, we have
\[
||\cP - \cD||_1 \leq 2^{k} \cdot \sqrt{\sum_{a \in \zerooneset^k} \ld(\widehat{\cD|_J}(a) - \widehat{\cP|_J}(a)\rd)^2}.
\]
\end{lemma}
\begin{proof}
First, we have
\begin{align*}
||\cD - \cP||_1 &= \sum_{x \in \boolset^n} |\cD(x) - \cP(x)| \\
&\leq 2^{n/2} \cdot \sqrt{\sum_{x \in \boolset^n} (\cD(x) - \cP(x))^2} \\
&= 2^{n} \cdot \sqrt{\Exu{x \sim \cU_n}{(\cD(x) - \cP(x))^2}} \\
&= 2^n \cdot \sqrt{\sum_{a \in \zerooneset^n} (\hD(a) - \hP(a))^2}
\end{align*}
where the first inequality holds by Cauchy-Schwarz.
Then, using \Cref{lem:k-junta closeness through projections} and \Cref{lem:fourier coefficients of marginals} we obtain the second claim.
\end{proof}

\begin{lemma}[Estimating a single Fourier coefficient]
\label[lemma]{lem:estimating a single fourier coeff}
There exists an algorithm that takes as input $m = O(\varepsilon^{-2} \log \delta^{-1})$ i.i.d. samples from a distribution $\cD$ over $\boolset^n$, $J \subseteq [n]$ and outputs $z$ such that $z = \hD(J) \pm \varepsilon \cdot 2^{-n}$ with probability at least $1-\delta$.
\end{lemma}
\begin{proof}
We have $\hD(J) = \Exui{x \sim \cD}{2^{-n} \cdot \chi_{J}(x)}$ and $\Exui{x \sim \cD}{(2^{-n} \cdot \chi_{J}(x))^2} = 2^{-2n}$.
The desiderata follows by standard concentration bounds applied to the average $z= \frac{1}{m} \cdot 2^{-n} \cdot (\chi_{J}(x_1) + \dots + \chi_{J}(x_m))$ where $x_1 \dots x_m$ are i.i.d. samples from $\cD$.
\end{proof}

Finally, we prove \Cref{lem:reduction from junta distribution to parity distribution}.

\begin{proof}[Proof of \Cref{lem:reduction from junta distribution to parity distribution}]
Here we design the algorithm $\cA'$ that learns an arbitrary $k$-junta distribution $\cD$ with set of relevant variables $J^\star$.
denote by $\cJ$ the set of all $J \subseteq [n]$ such that $|\hD (J)| \geq \varepsilon \cdot 2^{-( n + k/2)}$. We have $\cJ \subseteq 2^{J^\star}$, which implies $|\cJ| \leq 2^k$. 
We will show a procedure $\findheavyfourier$ (Algorithm \ref{alg:find heavy fourier}) such that 
the following holds.

\begin{lemma}
\label[lemma]{lem:findheavyfourier}
For each $J \in \cJ$, $\findheavyfourier$ (Algorithm \ref{alg:find heavy fourier}) returns $J$ with probability $ \Omega(2^{-k})$. 
Moreover, whenever $\findheavyfourier$ returns $(J, z)$ then both (i) $\hD(J) \neq 0$, and (ii) $z = \hD(J) \pm \varepsilon \cdot 2^{-( n + k/2)}$ hold with probability at least $ 1 - \poly(2^{-k})$. Finally, $\findheavyfourier$ uses $\poly(2^k, \varepsilon^{-1}) + T\ld(n, k, O(2^{k/2})\rd)$ time.
\end{lemma}
We defer the proof of \Cref{lem:findheavyfourier} and complete our reduction assuming \Cref{lem:findheavyfourier}.
Let $\cI$ be the set (not a multi-set) of sets $J$ returned by $\findheavyfourier$ after re-running it $O(k 2^{2k})$ times.
By coupon collector, $\cJ \subseteq \cI \subseteq 2^{J^\star}$ with high probability. Additionally, for all $J \in \cI$, $z_J$ satisfies $z_J = \hD(J) \pm \varepsilon \cdot 2^{-( n + k/2)}$ with high probability (by union bound).
Define $\cZ := \sum_{J \in \cI} z_J \cdot \chi_J$ and notice that $\cZ$ is a junta with relevant variables $L := \bigcup \cI \subseteq J^\star$. By \Cref{lem:tv distance Fourier bound} and \Cref{lem:fourier coefficients of marginals} we have
\begin{align*}
||\cZ - \cD||_1 &= \\
||\cZ|_{J^\star} - \cD|_{J^\star}||_1 &\leq \\
2^{k} \cdot \sqrt{\sum_{J \subseteq J^\star} (\widehat{\cD|_{J^\star}}(J) - \hZ|_{J^\star}(J))^2} &\leq \\
2^{k} \cdot \sqrt{\sum_{J \subseteq J^\star} 2^{2(n-k)} \cdot  (\widehat{\cD}(J) - \hZ(J))^2} &\leq \\
2^{k} \cdot \sqrt{2^k \cdot 2^{2(n-k)} \cdot \varepsilon^2  \cdot 2^{-2(n + k/2)}} &\leq \varepsilon.
\end{align*}
Finally, $\cZ$ might not be a probability distribution. We would like to ``round'' $\cZ$ to a probability distribution within $\ell_1$-distance $O(\varepsilon)$ from $\cD$. Since there exists a probability distribution $\cD$ satisfying $\lpnorm{1}{\cD-\cZ} \leq \varepsilon$, then such rounding exists. Here we show how to implement such rounding efficiently.

Since $\cZ$ is a junta with relevant variables $L \subseteq J^\star$, then we can simply round $\cZ|_L$ and extend it to a rounding of $\cZ$ by taking the product distribution with $\cU_{n -\ell}$, where $\ell = |L|$.
First, for all $x \in \boolset^\ell$ such that $\cZ|_L(x) < 0$ we set $\cZ|_L(x) \leftarrow 0$.
Then, we normalize all entries of $\cZ|_L$ so that $\sum_{x \in \boolset^\ell} \cZ|_L(x) = 1$.
It is straightforward to verify that both operations at most double the $\ell_1$ distance between $\cZ$ and $\cD$. Hence, after rounding $\cZ$ still satisfies $\lpnorm{1}{\cZ - \cD} = O(\varepsilon)$.

\paragraph*{Complexity.}
We are left to bound the computational 
complexity of our reduction.
We make $\poly(2^k)$ calls to $\findheavyfourier$. Each call uses $\poly(2^k, \varepsilon^{-1}) + T\ld(n, k, O(2^{k/2})\rd)$ time.
All other operations combined use at most $\poly(2^k, n)$ time.

\begin{algorithm} 
\textbf{Input:} Samples from $\cD \in \Delta(\boolset^n)$ \\
\textbf{Output:} $J \in \cJ$ and $y_J = \hD(J) \pm \varepsilon \cdot 2^{-(n+k/2)}$ 

\hrulefill

 Let $m = k + 1$ \\
 Sample a uniformly random matrix $A \in \zerooneset_2^{m\times n}$ \\
 Sample $x_1 \dots x_S \sim \cD$ i.i.d. \\
 Sample $q_1 \dots q_S \sim \cU_m$ i.i.d. \\
 Let $\chi_\emptyset+ y_J \cdot \chi_{J}$ be the noisy parity distribution returned by $\cA(x_1 \oplus A^\top q_1, \dots , x_S \oplus A^\top q_S)$, where $\cA$ is PAC if its input has noise of rate $1 - O(2^{-k/2})$. \\
 Run the algorithm from \Cref{lem:estimating a single fourier coeff} on $\cD$, $J$ and let $z$ be the returned estimate. Set parameters so that $z = \hD(J) \pm \frac{\varepsilon}{3} \cdot 2^{-( n + k/2)}$ with probability $\geq 1 - \poly(2^{-k})$. \\
 If $|z| \geq \frac{2 \varepsilon}{3} \cdot 2^{-(n+k/2)}$ \Return $J$ and $y_J$ \\
 
\caption{$\findheavyfourier$ \label{alg:find heavy fourier}}
\end{algorithm}

In the end, we show that $\findheavyfourier$ satisfies the properties stated in \Cref{lem:findheavyfourier}.

\begin{proof}[Proof of \Cref{lem:findheavyfourier}]
Fix $J \in \cJ$.
Define the event $\cE_J$ that \Cref{eq:desiderata only one fourier survives} holds, which happens with probability at least $2^{-m-1} \geq 2^{-(k+2)}$ by \Cref{lem:only one fourier survives}.
Conditioned on $\cE_J$, $\cD_A = \chi_\emptyset + \hD(J) \cdot \chi_J$ is a noisy junta distribution with relevant variables $J$.
Moreover, the samples $x_1 \oplus A^\top q_1 \dots x_S \oplus A^\top q_S$ are distributed according to $\cD_\cA$ (\Cref{lem:adding linear noise}).

Let $\cE_\cA$ be the event that $\cA$ is correct.
Conditioned on $\cE_J$, the noise rate of $\cD_A$ is $\eta_J =\frac 1 2 \cdot (1 - 2^n \cdot \hD(J))$, thus $J\in \cJ$ implies $1- 2\eta_J \geq \varepsilon \cdot 2^{-k/2}$.
Since, conditioned on $\cE_J$, the noise rate of $\cD_A$ is appropriate we have $\Pr{\cE_\cA \,|\, \cE_J} = \Omega(1)$.

From now on, condition on $\cE_\cA \cap \cE_J$.
Suppose that $\cA$ returns the noisy parity $\chi_\emptyset + z \cdot \chi_J$. 
Since $J \in \cJ$, $|\hD(J)| \geq \varepsilon \cdot 2^{-(n + k/2)}$ so $z = \hD(J) \pm \frac{\varepsilon}{3} \cdot 2^{-(n +k/2)}$ implies $|z| \geq \frac{2 \varepsilon}{3} \cdot 2^{-(n + k/2)}$.
Thus, $\findheavyfourier$ returns $(J, z)$ correctly with probabability $\Pr{\cE_\cA \cap \cE_J} = \Omega( 2^{-k})$.
Finally, points (i) and (ii) hold because of correctness of \Cref{lem:learning one fourier coefficient}.

Now we analyze the running time 
of $\findheavyfourier$.
$\cA$ uses $T\ld(n, k, O(2^{k/2})\rd)$ time.
The algorithm from \Cref{lem:learning one fourier coefficient} uses $\poly(2^k, \varepsilon^{-1})$ 
time.
All other operations combined use $\poly(n, 2^k)$ time.
\end{proof}
\end{proof}

\subsection{Reducing Parity Functions with Noise to Junta Distributions}
\label{sec:reducing noisy parity function to junta distributions}
In this section we prove the next lemma.

\begin{lemma}
\label[lemma]{lem:reduction from noisy parities to junta distributions}
Let $\cA$ be an algorithm that learns $k$-junta distributions over $\boolset^n$. Suppose that $\cA$ uses $T\ld(n, k\rd)$ time.
Then, there exists an algorithm $\cA'$ that  learns $k$-parity functions on $\boolset^n$ with noise rate $\eta < 1/2$ under the uniform distribution. Moreover, $\cA'$ uses at most $\poly(2^k, (1-2\eta)^{-1}) \cdot T\ld(n, k\rd)$ time.
\end{lemma}

\begin{proof}
Let $\chi_S$ be the parity function that we would like to learn. We observe samples of the form $(x, y)$ where $x \sim \cU_n$ and $y\sim \chi_S(x) \oplus \ber(\eta)$, where $\ber(\eta)$ is a Bernoulli random variable of parameter $\eta$. Namely, $y$ equals $\chi_S(x)$ flipped with probability $\eta$. Equivalently, we can write
\begin{equation*}
    y = \begin{cases}
			\ber(1/2), & \text{with probability $2\eta$ }\\
            \chi_S(x), & \text{with probability $1-2\eta$.}
		 \end{cases}
\end{equation*}

We construct a sample $z$ from a junta distribution $\cD$ as follows.
Given $(x, y)$ sampled as above, if $y = 1$ set $ z\leftarrow x$, otherwise sample a new independent copy of $(x, y)$. 
We have $z \sim \cD$ where $\cD$ is defined as the convex combination 
\[
\cD := \frac{1- 2\eta}{1 + 2\eta} \cdot \cP + \frac{4\eta}{1 + 2\eta} \cdot \cU_n
\]
and $\cP$ is the uniform distribution over $\{x \in \boolset^n \,|\, \chi_S(x) = 1\}$.
In fact, if $y = \ber(1/2)$ we have $z \sim U_n$ and if $y = \chi_S(x)$ then with probability $1/2$ $y =\chi_S(x) = 1$ and we have $z \sim \cP$ and with probability $1/2$ we reject. The probability of \emph{not} rejecting is $1 - \frac{1}{2}\cdot (1-2\eta)$, matching the coefficients above.
Notice that the Fourier representation of $\hD$ depends on whether: (i)  $S \neq \emptyset$ or (ii)$S= \emptyset$. 
In case (i) we have
\begin{equation*}
    \hD(A) = 2^{-n} \cdot \begin{cases}
			1-2\eta, & \text{if $A = S$}\\
            2\eta, & \text{if $A = \emptyset$} \\
            0, & \text{otherwise.}
		 \end{cases}
\end{equation*}
In case (ii) we have
\begin{equation*}
    \hD(A) = 2^{-n} \cdot \begin{cases}
            1, & \text{if $A = S = \emptyset$} \\
            0, & \text{otherwise.}
		 \end{cases}
\end{equation*}

\paragraph*{The algorithm}
Here we describe our algorithm $\cA'$.
Assume that we have a procedure $\findset(\bar\eta)$ that takes as input a guess for the noise rate and returns a candidate $\tilde S$ for $S$.
First, we design a procedure to certify $\tilde S = S$.
Draw $m = \poly(2^k, (1-2\bar\eta)^{-1})$ samples $(x_i, y_i)$ where $x_i \sim \cU_n$ and $y_i$ equals $\chi_S(x_i)$ corrupted with noise of rate $\eta$. Define
\[
\cI_{\tilde S} := \frac{1}{m}\cdot \sum_{i\in [m]} \chi_{\tilde S}(x_i) \cdot y_i
\]
Since $\Ex{(\chi_{\tilde S}(x_i) \cdot y_i)^2} = 1$, applying standard concentration inequalities we obtain, with high probability,
$\cI_{\tilde S} = \Ex{\chi_{\tilde S}(x_i) \cdot y_i}  \pm o(1-2\bar\eta)$.
Moreover, 
\[
\Ex{\chi_{\tilde S}(x_i) \cdot y_i} = 
\begin{cases}
		1-2\eta &\text{if $\tilde S = S \neq \emptyset$} \\
       1 &\text{if $\tilde S = S = \emptyset$} \\
        0 &\text{if $\tilde S \neq S$.}
	\end{cases}
\]
Therefore whenever we observe $\cI_{\tilde S} >  2 \cdot (1-2\bar \eta)$, then we can certify that $\tilde S = S$. 
Moreover, if we have $1-2\bar\eta \leq \frac{1}{4} \cdot (1-2\eta)$ and $\tilde S = S$ then $\cI_{\tilde S} > 2 \cdot (1-2\bar \eta)$ and we can certify $\tilde S = S$.
Now, we describe a procedure $\findset(\bar\eta)$ such that whenever $1-2\bar\eta \leq 1-2\eta$ returns a correct candidate $\tilde S = S$.
These two ingredients together make it possible to perform an exponential search over the parameter $1- 2\bar \eta$ by shrinking it in half at each round.
As soon as $1-2\bar\eta \leq \frac{1}{4} \cdot (1-2\eta)$, $\findset(\bar \eta)$ returns a correct $\tilde S = S$ and $\cI_{\tilde S} > 2 (1-2\bar\eta)$.
Whenever it happens, we return $\tilde S$.
Notice that the exponential search ensures that we obtain the right complexity in terms of $1-2\eta$.

Here we describe $\findset(\bar \eta)$. Recall that we only need $\findset$ to be correct if $1-2\bar\eta \leq 1-2 \eta$.
First, we run $\cA$  with error parameter $\varepsilon \cdot (1-2\bar\eta)$ on the samples $z \sim \cD$ generated as described above.
$\cA$ returns $\cF$ such that $||\cD - \cF||_1 \leq \varepsilon \cdot (1-2\bar\eta)$.
Recall that $\cA$ returns an \emph{evaluator}, so we have \emph{query access} to $\cF$. 

Assume $S \neq \emptyset$.
For each $x \in \boolset^n$ we have $\cD(x) = 2\eta \cdot 2^{-n} + (1 - 2\eta) \cdot 2^{-n+1} := \alpha$ if $\chi_S(x) =1 $ and 
$\cD(x) = 2\eta \cdot 2^{-n} := \beta$ if $\chi_S(x) = -1$. Notice that $\beta - \alpha = (1 - 2\eta) \cdot 2^{-n+1}$. If we have $S=\emptyset$, then $\cD(x) = 2^{-n}$ for all $x \in \boolset$.
We define a function $f:\boolset^n \rightarrow \boolset$ by rounding $\cF$.
If $\cF(x)$ is closer to $\beta$ than $\alpha$ then we set $f(x) = -1$, otherwise we set $f(x) = 1$.

A simple charging argument shows that whenever $f(x) \neq \chi_S(x)$ we can charge at least $(\beta - \alpha) / 2$ to $\lpnorm{1}{\cF - \cD}$. 
Therefore, $\lpnorm{1}{\cF - \cD} \leq \varepsilon \cdot (1-2\bar\eta) \leq \varepsilon \cdot (1-2\eta)$ implies $\lpnorm{1}{f - \chi_S} \leq \varepsilon \cdot 2^n$.
Finally, we use the algorithm from \Cref{lem:learn with queries} to find $\tilde S$. 

Unfortunately, this rounding scheme does not work when $S = \emptyset$ because $\cD(\cdot)$ is constant. To obviate this issue, we modify $\findset$ so that it returns two set at each round: $\tilde S$ and $\emptyset$. In this way, the run our certifying procedure both on $\tilde S$ and $\emptyset$. If either  $S = \tilde S$ or $S = \emptyset$, the certifying procedures detects that $\cI_{\tilde S}$ (respectively $\cI_{\emptyset}$) is larger than $2\cdot (1-2\bar\eta)$.

The complexity of each step of our reduction is $\poly(n, (1-2\bar\eta)^{-1}) \cdot T(n, k)$. Indeed, boosting the accuracy of $\cA$ to $\varepsilon \cdot (1-2\bar\eta)$ (in $\ell_1$) costs $\poly(\varepsilon^{-1}, (1-2\bar\eta)^{-1})$, the algorithm in \Cref{lem:learn with queries} runs in time $\poly(n)$ and all other operations combined cost at most $\poly(n)$.
Finally, our reduction uses time $\poly(n, (1-2\eta)^{-1}) \cdot T(n, k)$ thanks to our exponential search. 

\end{proof}
\section{Algorithm}
\label{sec:algorithm}

A key ingredient of our algorithm is a tolerant tester for distribution identity that is efficient in the high-confidence regime. First, we develop the tester and then we apply it to learning junta distributions.

\subsection{High-Confidence Tolerant Identity Tester}
\label{sec:high-confidence tolerant tester}
The goal of this section is to develop a sample-efficient tolerant identity tester for the high-probability regime. In particular, boosting off-the-shelf identity testers to failure probability $\delta$ incurs a multiplicative overhead of $\log \delta^{-1}$ on the number of samples. We improve over standard failure-probability boosting and obtain an algorithm that only pays an \emph{additive} overhead of $\log \delta^{-1}$.
To this end, we mimic the techniques used in \cite{diakonikolas2018sample}. 

\begin{definition}[Tolerant Identity Tester]
\label{def:tolerant identity tester}
We say that an algorithm $T$ is an $\ell_1$ $(\varepsilon_1, \varepsilon_2)$-tolerant identity tester for the distribution $\cP$ with sample complexity $m$ and failure probability $\delta$ if the following holds. 

For any distribution $\cD$, let $x_1 \dots x_m$ be i.i.d. samples from $\cD$.
\begin{itemize}
    \item If $||\cD - \cP||_1 \leq \varepsilon_1$ then $T(x_1 \dots x_m)$ returns ``close'' with probability at least $1-\delta$;
    \item If $||\cD - \cP||_1 \geq \varepsilon_2$ then $T(x_1 \dots x_m)$ returns ``far'' with probability at least $1-\delta$.
\end{itemize}
\end{definition}

Recall the ``Bernstein'' version of the standard bounded differences (McDiarmid) inequality.
\begin{lemma}[Bernstein version of McDiarmid's inequality~\cite{ying2004mcdiarmid}]\label[lemma]{lem:mcdiarmid}
Let {$Y_1,\dots, Y_m$} be independent random variables taking values in the set $\mathcal{Y}$. 
Let $f: \mathcal{Y}^m \rightarrow \bbR$ be a function of {$y_1, \dots ,y_m$} so that for every {$j \in [m]$}
and $y_1, \ldots, y_m, y_j^\prime \in \mathcal{Y}$, we have that:
\[
 \vert f(y_1,\ldots, y_j, \ldots, y_m) - f(y_1, \ldots, y_j^\prime, \ldots, y_m) \vert \leq B \;.
\]
Then, we have:
\begin{align}
\Pr{f(Y_1, \ldots, Y_m)-\E[f] \geq z } \leq \exp\left(\frac{-2z^2}{{m} B^2}\right) \;. \nonumber 
\end{align}
\end{lemma}

Now we can prove the main result of this section.

\begin{lemma}[High-confidence tolerant identity tester]
\label[lemma]{lem:high-confidence tolerant identity tester}
For any $\alpha, \epsilon, \delta> 0$ and any distribution $\cD$ with support $[s]$, there exists an $\ell_1$ $(\alpha, \alpha + \varepsilon)$-tolerant identity tester for $\cD$ with sample complexity $O(\varepsilon^{-2} (s + \log \delta^{-1}))$.
\end{lemma}
\begin{proof}
We define the same estimator used in \cite{diakonikolas2018sample} and we analyze its mean and concentration similarly to \cite{diakonikolas2018sample}.
Let $m = O(\varepsilon^{-2} (s + \log \delta^{-1}))$ and $Y_1 \dots Y_m$ be i.i.d. samples from $\cD$.
Let $X_1 \dots X_s$ be the histograms induced by $Y_1 \dots Y_m$, namely $X_i$ is the number of $j \in [m]$ such that $Y_j = i$.
Given a distribution $\cP$ supported on $[s]$, define 
\[
S_\cP(X) := \lpnorm{1}{\frac{X}{m} - \cP} = \sum_{i \in [s]} \abs{\frac{X_i}{m} - \cP(i)}.
\]

Let $\cP$ be a probability distribution supported on $[s]$.
In the following we prove that with probability $1 - \delta$ we have
\[
S_\cP(X) = \lpnorm{1}{\cP - \cD} \pm \frac \varepsilon 2.
\]
Then, to obtain the desired tolerant tester, we return ``close'' if $S_\cP \leq \alpha + \varepsilon / 2$ and return ``far'' otherwise.
By triangle inequality, we have
\[
\lpnorm{1}{\frac{X}{m} - \cP} = \lpnorm{1}{\cD - \cP} \pm \lpnorm{1}{\frac{X}{m} - \cD}.
\]
So, we are left proving that $\Pr{S_\cD(X) > \varepsilon / 2} \leq \delta$.
First, we  bound the expectation of $S_\cD(X)$.
\begin{align}
\Ex{S_\cD(X)} &= \sum_{i \in [s]} \Ex{\abs{\frac {X_i}  m - \cD(i)}} \nonumber \\
&\leq \sum_{i \in [s]} \sqrt{\Var{\frac{X_i}{m} - \cD(i)}} \nonumber \\
&\leq \sum_{i \in [s]} \sqrt{\frac {m \cdot \cD(i)}  {m^2}} \nonumber \\
&\leq \sum_{i \in [s]} \sqrt{ \frac{m \cdot \frac 1 s} {m^2}}  \label{eq:symmetric minimizer}\\
&= \sqrt{\frac s m} \leq \frac \varepsilon 4. \nonumber
\end{align}
\Cref{eq:symmetric minimizer} holds because $\cD \mapsto \sum_{i \in [s]} \sqrt{\cD(i)}$ is a concave function of $\cD$ and so it is minimized when $\cD(i) = 1/s$ for all $i \in [s]$.

Then, we show that $S_\cD(X)$ concentrates around its expectation. Notice that the map $(Y_1 \dots Y_m) \mapsto S_\cD(X)$ has Lipschitz constant $B = 1/m$, thus applying \Cref{lem:mcdiarmid} we obtain
\[
\Pr{\abs{S_\cD(X) - \Ex{S_\cD(X)}} > \frac \varepsilon 4} \leq \exp\ld(\frac{-2\cdot \ld(\frac{\varepsilon}{4}\rd)^2}{m \cdot B^2}\rd) =\exp\ld(\Omega\ld(\varepsilon^2 \cdot m\rd)\rd) \leq \delta.
\]
Therefore, $\Pr{S_\cD(X) > \varepsilon / 2} \leq \delta$.

\end{proof}

\subsection{An Algorithm for Learning Junta Distributions}
In this section we prove \Cref{thm:main thm statistical}, restated below.
\MainThmStatistical*
First, we prove a few technical lemmas. 

\begin{lemma}[Learning Fourier coefficients]
\label[lemma]{lem:learning fourier coefficients}
There exists an algorithm that takes as input $O(\varepsilon^{-2}(k + \log \delta^{-1}))$ i.i.d. samples from a distribution $\cD$ over $\boolset^k$ and outputs a distribution $\cQ$ over $\boolset^k$ such that, for each $J \subseteq [k]$ 
\[
\hQ(J) = \hD(J) \pm \frac{\varepsilon}{2^{k}}. 
\]
Such an algorithm succeeds with probability at least $1-\delta$ and runs in time $\poly(2^k, \varepsilon, \log \delta^{-1})$.
\end{lemma}
\begin{proof}
Let $m = O(\varepsilon^{-2} 2^k (k + \log \delta^{-1}))$ and $x_1 \dots x_m$ be i.i.d. samples from $\cD$. For each $J \subseteq \boolset^k $, define $\hP(J)$ as the $J$ Fourier coefficient of the empirical distribution of the samples, namely 
\[
\hP(J) := 2^{-k} \cdot \sum_{i \in [m]} \frac 1 m \cdot \chi_J(x_i).  
\]
For each $J \subseteq [k]$, we have $\Ex{\hP(J)} = \hD(J)$ and 
\[
\Var{\hP(J)} = \frac {1} {m^2} \sum_{i \in [m]} \Var{2^{-k} \cdot \chi_J(x_i)} = \frac {2^{-2k}} m \cdot \Var{\chi_S(x_1)} \leq \frac {2^{-2k}} m.
\]
Thus, by standard concentration bounds we have that, for each $J \subset [k]$, 
\[
\Pr{\abs{\hP(J) - \hD(J)} > \varepsilon \cdot 2^{-k}} \leq 2^{-k} \cdot \delta.
\]
By taking a union bound over all $J\subseteq [k]$ we obtain:

\begin{equation} \label{eq:all fourier are well approximated}
\Pr{\max_{J \subseteq [k]}\abs{\hP(J) - \hD(J)} > \varepsilon \cdot 2^{-k}} \leq \delta.
\end{equation}

The values $\hP$ satisfy the desiderata, however we might have that $\cP$, the inverse Fourier transform of $\hP$, is not a probability distribution.
On the other hand, conditioning on the previous step succeeding, we know that there exists a probability distribution $\cQ$ such that $\lpnorm{\infty}{\hP - \hD} \leq \varepsilon \cdot 2^{-k}$.
We write a linear program that searches for a distribution $\cQ$ such that
\begin{align*}
& \lpnorm{\infty}{\hP - \hQ} \leq \varepsilon \cdot 2^{-k} \\
& \sum_{x \in \boolset^k} \cQ = 1 \\
& \cQ(x) \geq 0 \; \text{ for all } \; x \in \boolset^k.
\end{align*}
Since each coefficient of $\hQ$ is a linear function of $\cQ$ we can solve the above LP in time $\poly(2^k)$. Triangle inequality and \Cref{eq:all fourier are well approximated} imply our desiderata.
\end{proof}

\begin{lemma}
\label[lemma]{lem:from fourier estimate to l1 estimate}
Let $\cQ$ and $\cD$ be distributions over $\boolset^k$ with $\max_{a\in \boolset^k} |\hQ(a) - \hD(a)| \leq \varepsilon \cdot 2^{-3k/2}$, then $||\cQ - \cD||_1 \leq \varepsilon$.
\end{lemma}
\begin{proof}
\begin{align*}
||\cQ - \cD||^2_1 &\leq 2^{k} \cdot ||\cQ - \cD||_2^2 \\ 
&= 2^{2k} \cdot \sum_{a \in \zerooneset^k} (\hQ(a) - \hD(a))^2 \\
&\leq 2^{2k} \cdot \sum_{a \in \zerooneset^k} \frac{\varepsilon^2}{2^{3k}}\\
&\leq \varepsilon^2.
\end{align*}
\end{proof}

\begin{observation} \label{obs:estimating fourier leads to l1}
Combining \Cref{lem:from fourier estimate to l1 estimate} and \Cref{lem:learning fourier coefficients} we obtain the following.
\label{cor:l1 guarantee from learning fourer coefficients}
The algorithm from \Cref{lem:learning fourier coefficients} run with parameter $\varepsilon' = \varepsilon \cdot 2^{-k/2}$ learns $\cD$ in $\ell_1$ norm up to additive error $\varepsilon$.
\end{observation}

\begin{algorithm}
\textbf{Input:} $m = O\left(\frac{k}{\varepsilon^2} (2^k + \log n)\right)$ i.i.d. samples from a distribution $\cP$ over $\boolset^n$. Moreover, $\cP$ is promised to be a junta distribution over an unknown size-$k$ set $J^\star \subseteq [n]$.\;

\textbf{Output:} The explicit description of a distribution $\cQ$ over $\boolset^n$ such that $||\cQ - \cP||_1 = O(\varepsilon)$ \;

\hrulefill

Let $\tester$ be the $(2\varepsilon, 3\varepsilon)$-tolerant identity tester of \Cref{lem:high-confidence tolerant identity tester} with failure probability $n^{-2k}$\\
Let $\learnfouriercoefficients$ be the algorithm described in \Cref{lem:learning fourier coefficients} with error parameter $\varepsilon' = \varepsilon \cdot 2^{-k/2}$ and  failure probability $k^{-2}$\\
Initialize $N \leftarrow \emptyset$\\
Initialize $\cQ \leftarrow \cU_n$ \\
\For{$i \in [k]$}{
\If{For all $S\subseteq [n]$ with $S \supseteq N$ and $|S| = k$, $\tester(\cP|_{S}, \cQ|_{S})$ returns ``close''}{
\Return $\cQ$ \label{lst:return Q} \label{lst:early return of Q}
}
Let $S$ be such that $\tester(\cP|_{S}, \cQ|_{S})$ returns ``far''\;
$\cR \leftarrow \learnfouriercoefficients(\cP|_{S})$\;
For each $A \subseteq S$ with $|\hR(A)| > \varepsilon \cdot 2^{-3k/2}$ do $N \leftarrow N \cup A$ \label{lst:augment N}\;
Let $\cQ$ be the junta distr. on variables $N$ s.t. $\cQ|_N \leftarrow \learnfouriercoefficients(\cP|_N)$ \label{lst:update Q}\;
}
\Return $\cQ$ \label{lst:late return of Q}

\caption{Junta Distribution Learner \label{alg:junta learner}}
\end{algorithm}

\paragraph*{Analysis of Algorithm \ref{alg:junta learner}}
Here we prove \Cref{thm:main thm statistical}. Namely, that Algorithm \ref{alg:junta learner} correctly returns an approximation of $\cP$ with high probability, and that it uses at most $O\left(\frac{k}{\varepsilon^2} (2^k + \log n)\right)$ samples and $O(\min(2^n, n^k) \cdot \frac{k}{\varepsilon^2} (2^k + \log n))$ time.

Denote with $J^\star$ the set of relevant variables of $\cP$. Let $\cE$ be the event that $\learnfouriercoefficients$ and $\tester$ never fail throughout the execution of Algorithm \ref{alg:junta learner}.

\begin{lemma}
$\Pr{\cE} = o(1)$.
\end{lemma}
\begin{proof}
The failure probability of $\tester$ is set to $n^{-2k}$ and we make at most $k \cdot n^k$ calls to it.
The failure probability of $\learnfouriercoefficients$ is set to $k^{-2}$ and we make at most $2 k$ calls to it.
\end{proof}

\begin{lemma}
\label[lemma]{lem:N contained in J}
Conditioned on $\cE$, we have $N \subseteq J^\star$ throughout the execution of Algorithm \ref{alg:junta learner}.
\end{lemma}
\begin{proof}
We initialize $N \leftarrow \emptyset$ and the only step where we add elements to $N$ is line \ref{lst:augment N}. Conditioned on $\cE$ we have $|\hR(A) - \widehat \cP|_S| \leq \varepsilon \cdot 2^{-3k/2}$, thus $|\cR(A)| > \varepsilon \cdot 2^{-3k/2}$ implies $\widehat \cP|_S(A) \neq 0$ and, in turn, $\hP(A) \neq 0$. Hence, $A \subseteq J^\star$.
\end{proof}

\begin{lemma} \label[lemma]{lem:error of learning Q on N}
Conditioned on $\cE$, let $\cQ$ and $N$ be defined as in line \ref{lst:update Q}. Then,
\[
|\widehat {\cP|_N}(A) - \widehat {\cQ|_N}(A)| \leq \varepsilon \cdot 2^{-(|N| + k/2)}.
\]
\end{lemma}
\begin{proof}
Notice that $\cP|_N$ is supported on $\boolset^{|N|}$ and apply \Cref{lem:learning fourier coefficients} with error parameter $\varepsilon' = \varepsilon \cdot 2^{-k/2}$.
\end{proof}

\begin{lemma}
\label[lemma]{lem:N increases}
Conditioned on $\cE$, the size of $N$ increases by at least one at line \ref{lst:augment N}.
\end{lemma}
\begin{proof}
Since $\tester(\cP|_{S}, \cQ|_{S})$ returned ``far'', we have that $||\cP|_{S} - \cQ|_{S}||_1 >  2\varepsilon$. Thus, by \Cref{lem:from fourier estimate to l1 estimate}, there exists $A \subseteq S$ such that $|\widehat {\cP|_{S}}(A) - \widehat {\cQ|_{S}}(A)| > 2 \varepsilon \cdot 2^{-3k/2}$. Moreover, for all $A \subseteq N$ we have 
\[
|\widehat {\cP|_{S}}(A) - \widehat {\cQ|_{S}}(A)| = 2^{|N| - k} \cdot |\widehat {\cP|_N}(A) - \widehat {\cQ|_N}(A)| \leq \varepsilon \cdot 2^{-3k/2}
\]
where the first equality holds by \Cref{lem:fourier coefficients of marginals} and the second inequality holds by \Cref{lem:error of learning Q on N}.
Thus, there must exist have $A \not \subseteq N$ such that 
\[
|\widehat {\cP|_{S}}(A)| = |\widehat {\cP|_{S}}(A) - \widehat {\cQ|_{S}}(A)| > 2 \varepsilon \cdot 2^{-3k/2}
\]
where the first equality holds because $\cQ$ is a junta with relevant variables $N$ and $S\not\subseteq N$.
Finally, since $\learnfouriercoefficients(\cP|_S)$ returns $\cR$ such that $\max_{B\subseteq S} |\widehat{\cP|_S}(B) - \hR(B)| \leq \varepsilon \cdot 2^{-3k/2}$ then we have $\hR(A) > \varepsilon \cdot 2^{-3k/2}$. So we increase the size of $N$ at line \ref{lst:augment N}.
\end{proof}

\begin{lemma}
Conditioned on $\cE$, Algorithm \ref{alg:junta learner} returns $\cQ$ such that $||\cQ - \cP||_1 \leq 2\varepsilon$.
\end{lemma}
\begin{proof}
If $\cQ$ is returned on line \ref{lst:early return of Q}, then for $S = J^\star \supseteq N$ we obtain 
\[
|\cP - \cQ| = |\cP|_{J^\star} - \cQ|_{J^\star}| \leq 2\varepsilon
\]
where the first inequality holds because $\cP$ is a junta wrt $J^\star$ and $\cQ$ is a junta wrt $N \subseteq J^\star$.

Suppose that $\cQ$ is returned on line \ref{lst:late return of Q}. Then, when executing line \ref{lst:late return of Q} we have $N \subseteq J^\star$ and $|N| = k$, thus $N = J^\star$.
Finally, \Cref{lem:error of learning Q on N} applied on the last iteration (i.e., when $N=J^\star$), together with \Cref{lem:from fourier estimate to l1 estimate} and \Cref{lem:k-junta closeness through projections} imply
\[
|\cP - \cQ| = |\cP|_N - \cQ|_N| \leq \varepsilon. 
\]
\end{proof}

\iftoggle{anonymous}{

}{
\section*{Acknowledgments}
I would like to thank Caleb Koch for the many insightful conversations about juntas that we had and for introducing me to learning theory in the first place.
}

\printbibliography
\end{document}